\documentclass[letterpaper]{article}
\usepackage{uai2020}
\usepackage[margin=1in]{geometry}

\usepackage{times}

\usepackage{natbib}
\usepackage{amsmath}
\usepackage{amsthm}
\usepackage{enumitem}
\usepackage{amsfonts}
\usepackage{mathtools}
\usepackage{graphicx}
\usepackage{caption}
\usepackage{algorithm,algorithmic}
\usepackage{dsfont}
\usepackage{subcaption}

\theoremstyle{definition}
\newtheorem{theorem}{Theorem}

\newtheorem{proposition}{Proposition}

\newtheorem{definition}{Definition}
\newtheorem{example}{Example}
\newtheorem{assumption}{Assumption}

\newcommand*{\R}{\mathbb{R}}
\newcommand{\G}{\mathcal{G}}
\newcommand{\E}{\mathbb{E}}

\newcommand{\bs}{\boldsymbol}

\newcommand{\indep}{\perp\!\!\!\!\!\!\perp}
\graphicspath{ {./figures/} }

\title{Model-Augmented Conditional Mutual Information Estimation \\ for Feature Selection}

\author{
{
{\bf Alan Yang}\textsuperscript{1}, 
{\bf AmirEmad Ghassami}\textsuperscript{1}, 
{\bf Maxim Raginsky}\textsuperscript{1}, 
{\bf Negar Kiyavash}\textsuperscript{2}, 
\normalfont{and} 
{\bf Elyse Rosenbaum}\textsuperscript{1}
} 
\\
\textsuperscript{1}Department of Electrical and Computer Engineering, University of Illinois at Urbana-Champaign (UIUC) \\
\textsuperscript{2}College of Management of Technology, \'Ecole Polytechnique F\'ed\'erale de Lausanne (EPFL)
\\
\texttt{
	\{asyang2, ghassam2, maxim, elyse\}@illinois.edu, negar.kiyavash@epfl.ch
}
} 

\begin{document}

\maketitle

\begin{abstract}
    Markov blanket feature selection, while theoretically optimal, is generally challenging to implement. This is due to the shortcomings of existing approaches to conditional independence (CI) testing, which tend to struggle either with the curse of dimensionality or computational complexity. We propose a novel two-step approach which facilitates Markov blanket feature selection in high dimensions. First, neural networks are used to map features to low-dimensional representations. In the second step, CI testing is performed by applying the $k$-NN conditional mutual information estimator to the learned feature maps. The mappings are designed to ensure that mapped samples both preserve information and share similar information about the target variable if and only if they are close in Euclidean distance. We show that these properties boost the performance of the $k$-NN estimator in the second step. The performance of the proposed method is evaluated on both synthetic and real data.
\end{abstract}

\vspace{-2mm}
\section{INTRODUCTION}
\vspace{-1mm}

A prominent approach for selecting features relevant to a target variable is to choose a minimal set that renders the target variable conditionally independent of the rest of the features. This approach is referred to as Markov blanket feature selection, and has attracted wide attention due to its theoretical optimality and strong performance \citep{AliferisSTMK2010}. However, testing for conditional independence (CI) from data is a notoriously challenging task, especially for continuous-valued variables of high dimension \citep{Paninski2003}. As a result, Markov blanket feature selection is challenging to implement in the case of high-dimensional features. For instance, time-series features with high Markov order, such as speech samples, are equivalent to high-dimensional variables.

One of the most promising approaches for (conditional) independence testing is based on the estimation of mutual information (MI) and conditional mutual information (CMI).
Due to its desirable properties, such as invariance under invertible transformations and ability to capture arbitrary nonlinear relationships \citep{CoverT2006}, MI commonly appears in several machine learning tasks, including feature selection \citep{Battiti1994}, clustering \citep{KraskovG2008}, and learning graphical models \citep{Koller2009}.
Numerous MI and CMI estimators have been proposed for continuous-valued variables. Most non-parametric estimators are based on $k$-NN density estimation \citep{KraskovSG2004,FrenzelP2007,GaoSOP2017} or kernel density estimation \citep{MoonKH2017}. Those estimators have been demonstrated to be fast and reliable for many problems with moderate dimension and sample size, yet have difficulty scaling to high dimensions and may exhibit high sample complexity \citep{GaoVG2015}.

Towards MI estimation in high dimensions, model-based approaches, which optimize variational bounds on MI, have recently been proposed \citep{BelghaziIASOYCH2018, PooleOVAT2019}. Model-based methods take advantage of the representational power of neural networks in order to discover complex structures in the variables. This can give them an advantage over non-parametric methods in high dimensions. Similar model-based approaches have also been proposed for CMI estimation \citep{MukherjeeAK2019}.
However, model-based methods require an individual neural network to be trained and tuned for each CI test performed, rendering those methods unsuitable for Markov blanket feature selection.

In this paper, we propose a model-augmented approach to CI testing suitable for Markov blanket feature selection. Our method simultaneously retains the computational tractability of non-parametric estimators and leverages the representational power of model-based methods. The main contributions of our work are as follows.
\begin{itemize}
\vspace{-0.5mm}
    \item We present an efficient Markov blanket feature selection algorithm for high-dimensional features based on CMI estimation. It efficiently implements the required CI tests by gradually increasing the conditioning set size (Section \ref{sec:MB_feature_selection}).

    \item We propose a two-step approach for efficiently estimating each of the required CMIs in the high dimensional case. In the first step, low-dimensional mappings of the features are learned by optimizing an objective with a likelihood and regularization term. In the second step, the $k$-NN CMI estimator is used to perform the CI tests on the mapped samples. 

    \item  Regarding the likelihood term in the objective, we show that it minimizes an upper bound on the CMI estimation error introduced by the mappings. Thus it preserves all the information required for accurate CMI estimation (Subsection \ref{subsec:mle_mappings}).

    \item We propose an efficient method for simultaneously learning the feature mappings for all CI tests that uses parameter sharing and a feature dropout technique that we call block-dropout (Subsection \ref{subsec:block_dropout}).

    \item Regarding the regularization term in the objective, we propose a novel regularization approach which ensures that the mappings satisfy the \emph{information efficiency} property: mapped samples share similar information about the target variable if and only if they are close in Euclidean distance. We show that the regularization term improves the performance of the $k$-NN CMI estimator in the second step of our approach (Section \ref{sec:info_efficiency_knn}).
\vspace{-0.5mm}
\end{itemize}

The proposed method is evaluated on both synthetic and real data in Section \ref{sec:experiments}. We applied our method to select time-series diagnostic features for predicting datacenter hard disk drive (HDD) failures. Due to the non-stationary nature of the data, existing methods either fail to eliminate all redundant features consistent with expert knowledge or miss many informative features.

\paragraph{Related Work.}
There is an extensive literature on non-parametric CI tests that use test statistics other than CMI. For example, kernel-based methods characterize CI in reproducing kernel Hilbert spaces \citep{GrettonFTSSS2007,LeeH2017}. Those methods have demonstrated some success for high-dimensional problems, but tend to scale poorly with sample size since large kernel matrices need to be computed. Model-based approaches have been proposed for the problem of general-purpose CI testing in high dimensions \citep{SenSSDS2017}. However, as is the case for model-based CMI estimators, the complexity of those methods renders them unsuitable for Markov blanket feature selection. In this work, we focus on the $k$-NN CMI estimator in the second step of our approach since it performs comparably to kernel-based methods \citep{Runge2018}, yet is computationally easier.

Our Markov blanket feature selection approach is similar to embedded feature selection methods, which attempt to efficiently consider relationships between features by performing feature selection and learning simultaneously. For example, the Lasso \citep{Tibshirani1996} selects features using $\ell_1$ regularization for regression; this concept has since been extended to more general learning-based feature selection \citep{LiCW2015}.
Other approaches attempt to quantify the effect of a feature on a learned model \citep{YeS2018} or learn a fixed-size set of features in an unsupervised setting \citep{AbidBZ2019}. Unlike those methods, our proposed approach directly seeks the minimal feature set for a particular prediction task using conditional independence information. 

\paragraph{Notation.}
We denote random variables by capital letters (e.g., $X$) and their realizations by small letters (e.g., $X=x$). We assume that all variables are jointly continuous, although the results can be extended to the case of discrete target variable $Y$ by appropriately replacing densities with discrete distributions.

Let $V=\{X_1,...,X_m\}$ denote the set of all the features and let $Y$ be the target variable. We denote the set of indices $=\{1,...,k\}$ for $1\le k\le m$ by $[k]$, and a subset of the features with indices in $S\subseteq[m]$ by $X_S$.

\vspace{-2mm}
 \section{AN EFFICIENT IMPLEMENTATION OF MARKOV BLANKET FEATURE SELECTION}
\vspace{-1mm}
\label{sec:MB_feature_selection}

In this Section, we first formally define Markov blanket feature selection and then present an efficient implementation of this approach in Subsection \ref{subsec:control_conditioning_set}.

Consider a set of features $V = \{X_1,...,X_m\}$ that are used to predict a target variable $Y$. The goal of feature selection is to find a minimal subset of the $X_i$'s that can be used to optimally predict $Y$. We focus on the case where $Y$ is low-dimensional, either initially or with the aid of dimensionality reduction. The features may be high-dimensional however, rendering direct application of the $k$-NN CMI estimator impractical.
\begin{definition}[Markov blanket]\label{def:markov_blanket}
A subset $X_{S_M}\subseteq V$ is a \textit {Markov blanket} of $Y$ if it is a minimal set such that $Y$ is conditionally independent of any subset $X_S\subseteq V$ given $X_{S_M}$. This is denoted as $Y\indep X_S\vert X_{S_M}$.
\end{definition}
Therefore, given the Markov blanket, the values of the rest of the variables become superfluous. 
Any approach which chooses relevant features in the sense of Definition \ref{def:markov_blanket} is referred to as Markov blanket feature selection.

Let $\G$ be the DAG corresponding to the Bayesian network of $V\cup \{Y\}$. $X_i$ is called a parent of $Y$ if there is an edge $X_i\rightarrow Y$ in $\mathcal{G}$; similarly, $X_i$ is a child of $Y$ if there is an edge $Y\rightarrow X_i$. Under some conditions (see Supplementary Materials), the Markov blanket of each variable is unique and consists of its parents, its children, and the other parents of its children (called the \emph{coparents}) \citep{Koller2009}. 

\vspace{-2mm}
\subsection{Controlling the Conditioning Set Size}\label{subsec:control_conditioning_set}
\vspace{-1mm}

From Definition \ref{def:markov_blanket} it is evident that a single CI test is sufficient to determine whether a feature belongs to the Markov blanket of the target variable; of course, this test is performed conditioned on all of the other features. In practice, testing for conditional independence is challenging when the conditioning set is large due to the curse of dimensionality. Therefore, in general one cannot hope to determine whether or not a particular feature is in the Markov blanket using only a single CI test.

Our feature selection approach provides an efficient method for identifying the Markov blanket, and is presented in Algorithm \ref{alg:markov_blanket}. In this algorithm, $\Delta$ denotes the maximum degree of the underlying Bayesian network, which in practice can be treated as a hyperparameter. The first part of the algorithm (lines 3-14) searches for the set $\textit{Adj}$, which is the set of all features adjacent to the target variable (parents and children). A feature is declared to be non-adjacent if a feature set rendering it conditionally independent of $Y$ is found. For each feature, we gradually increase the number of conditioning variables, starting with the empty set; this is similar to the approach of the PC algorithm for causal structure learning \citep{spirtes2000causation}. As is the case for PC, the computational complexity of this part of the algorithm is polynomial in the number of features.

As mentioned earlier, coparents of $Y$ also belong to the Markov blanket of $Y$. Coparents have the property that they are statistically dependent conditioned on any subset of variables which contains at least one of their common children \citep{pearl2014probabilistic}. Therefore, in order to find the coparents of $Y$ we perform a CI test conditioned on variables adjacent to $Y$ (lines 15-21). Finally, Algorithm \ref{alg:markov_blanket} returns the set of adjacents and coparents of $Y$.

\begin{algorithm}[t]
\caption{Markov Blanket Feature Selection}
\begin{algorithmic}[1]
\STATE {\bfseries input:} features $V$, target $Y$,
\STATE \qquad\quad max conditioning set size $\Delta$
\STATE {\textit{// Identify Adjacents (Parents and Children)}}
\STATE Initialize $Adj\gets V$
\FOR {$c \text{ from } 0 \text{ to } \Delta$}
    \FOR {$X_i\in Adj$}
    \FOR {$X_S\subseteq Adj\setminus\{X_i\}$ {such that} $|X_S| = c$}
        \IF {$Y\indep X_i | X_S$}
            \STATE $Adj \gets Adj\setminus \{X_i\}$
            \STATE \textbf{break}
        \ENDIF
    \ENDFOR
    \ENDFOR
\ENDFOR
\STATE {\textit{// Identify Co-parents}}
\STATE Initialize $CoP\gets \emptyset$.
\FOR{$X_i\in V\setminus Adj$}
    \IF {$Y\not\indep X_i|Adj$}
        \STATE $CoP \gets CoP\cup\{X_i\}$
    \ENDIF
\ENDFOR
\STATE \textbf{return} $Adj\cup CoP$
\end{algorithmic}
\label{alg:markov_blanket}
\end{algorithm}

In many applications, the target variable cannot be the \emph{cause} of any of the features in the system. 
This is always the case when features are time series and the target is a variable to be predicted at the end of the time horizon. One example is the HDD failure prediction problem considered in Section \ref{sec:experiments}.
In this case, Markov blanket feature selection is equivalent to causal feature selection, which is desirable for its robustness to shifts in the distributions  \citep{guyon2007causal}. 
This connection is explored in more detail in the Supplementary Materials. 

We use CMI as the test statistic for the CI tests in Algorithm \ref{alg:markov_blanket}. The following section introduces our proposed CMI estimation method.

\vspace{-2mm}
\section{MODEL-AUGMENTED CMI ESTIMATION}\label{sec:model_aug_cmi}
\vspace{-1mm}

For jointly continuous random variables $X$, $Y$, and $Z$, the CMI between $X$ and $Y$ given $Z$ is defined as
\begin{equation}\label{eq:conditional_mutual_information}
I(X;Y\vert Z) = \E \Bigg[ \log \frac{p(X,Y|Z)}{p(X|Z)p(Y|Z)} \Bigg],
\end{equation}
where $p(X,Y|Z)$, $p(X|Z)$, and $p(Y|Z)$ are the conditional densities and the expectation is taken over the joint distribution. When the conditioning on $Z$ is removed, \eqref{eq:conditional_mutual_information} is known as the mutual information (MI) between $X$ and $Y$, and is denoted by $I(X;Y)$.

For Markov blanket feature selection, Algorithm \ref{alg:markov_blanket} requires estimates of a collection of CMI values of the form $I(Y;X_i|X_S)$ for a feature $X_i\in V$ and conditioning set $X_S\subseteq V\setminus\{X_i\}$. We propose a two-step approach for estimating each of those CMI values.
\begin{enumerate}
    \item For each feature $X_i$, find a low dimensional mapping $f_i: X_i \mapsto f_i(X_i)$. 

    
    \item Estimate $I(Y;f_i(X_i)|F(X_S))$ as a proxy for the true CMI $I(Y;X_i|X_S)$ using the $k$-NN CMI estimator, where $F(X_S)\vcentcolon=\{f_j(X_j) : j\in S\}$.
\end{enumerate}

We propose a regularized maximum likelihood objective for learning the feature mappings in the first step. The analysis of the likelihood part of the objective and our proposed efficient implementation are presented in Subsections \ref{subsec:mle_mappings} and \ref{subsec:block_dropout}, respectively. The regularization part of the objective is introduced in Section \ref{sec:info_efficiency_knn}.

\vspace{-2mm}
\subsection{Learning Information-Preserving Feature Mappings}\label{subsec:mle_mappings}
\vspace{-1mm}

As mentioned earlier, we propose to use a maximum likelihood approach to learn the mappings. In this section we analyze the connection between the likelihood term and the mapping-induced error.

We first discuss the case of estimating the MI $I(Y;X_i)$ between the target variable and a single feature $X_i$; this is required for the unconditional independence tests in Algorithm \ref{alg:markov_blanket}.
We show that a maximum likelihood estimator minimizes an upper bound on the mapping-induced error $I(Y;X_i) - I(Y;f_i(X_i))$; note that this quantity is positive by the data-processing inequality \citep{CoverT2006}. 

Since the underlying distributions are unknown, for each $i\in [m]$ we introduce a distribution $q_{\theta_i}(Y|f_i(X_i))$, parameterized by $\theta_i$, which serves as a surrogate for the true conditional distribution $p(Y|f_i(X_i))$. In practice, the surrogate distributions are chosen to be in a parametric family. For example, $q_{\theta_i}$ may be chosen to be a normal distribution for continuous, univariate $Y$, where the mean and variance are functions of $f_i(X_i)$ parameterized by neural networks.

\begin{theorem}\label{theorem:mi_lower_bd}
The following minimizes an upper bound on the non-negative error $I(Y;X_i) - I(Y;f_i(X_i))$:
\begin{equation}\label{eq:mi_opt}
\max_{f_i,\theta_i}\, \E \big[ \log q_{\theta_i}(Y|f_i(X_i)) \big],
\end{equation}
where the expectation is taken over the joint distribution of $Y$ and $X_i$.
\end{theorem}
Intuitively, by maximizing the data log likelihood of a surrogate distribution $q_{\theta_i}(Y|f_i(X_i))$, $f_i$ is encouraged to be a sufficient statistic of $X_i$ for $Y$.

The following theorem shows a similar result for estimating the CMI $I(Y;X_i|X_S)$. For all $i\in[m]$ and $S\subseteq [m]\setminus\{i\}$, consider surrogate distributions $q_{\theta_i}(Y|f_i(X_i))$ and $q_{\theta_{i,S}}(Y|f_i(X_i),F(X_S))$ with respective parameters $\theta_{i}$ and $\theta_{i,S}$ for the true distributions $p(Y|f_i(X_i))$ and $p(Y|f_i(X_i),F(X_S))$.

\begin{theorem}\label{theorem:cmi_consistency}
The following minimizes an upper bound on the error $|I(Y;X_i|X_S)-I(f_i(X_i)|F(X_S))|$:
\begin{alignat}{2}\label{eq:cmi_opt}
  &\max_{\substack{\theta_i,\theta_{i,S}, \\ f_j : j\in \{i\}\cup S}}
     & \E\Big[&\log q_{\theta_i}(Y|f_i(X_i))\nonumber \\[-6mm]
  && &+ \log q_{\theta_{i,S}}(Y|f_i(X_i),F(X_S))\Big],
\end{alignat}
where the expectation is over the joint distribution of $Y$, $X_i$, and $X_S$.
\end{theorem}


For any CI test in Algorithm \ref{alg:markov_blanket}, we need to solve an optimization problem in the form of \eqref{eq:mi_opt} or \eqref{eq:cmi_opt} over a separate set of mappings and surrogate distributions.

\begin{example}\label{example:two_features}
In the case of $m=2$, we need to maximize 
\begin{itemize}[leftmargin=3mm]
    \item \hspace{-1mm}$\E[\log q_{\theta_1}(Y|f_1(X_1))]$ over $f_1$ and $\theta_1$,
    \item \hspace{-1mm}$\E[\log q_{\theta_2}(Y|f_2(X_2))]$ over $f_2$ and $\theta_2$,
    \item \hspace{-1mm}$\E[\log q_{\theta_1'}(Y|f_1'(X_1)) + \log q_{\theta_{1,\{2\}}}\!(Y|f_1'(X_1),f_2'(X_2))]$ over $f_1'$, $f_2'$, $\theta_1'$, and $\theta_{1,\{2\}}$, and
    \item \hspace{-1mm}$\E[\log q_{\tilde{\theta}_2}(Y|\tilde{f}_2(X_2)) + \log q_{\theta_{2,\{1\}}}\!(Y|\tilde{f}_2(X_2),\tilde{f}_1(X_1))]$ over $\tilde{f}_1$, $\tilde{f}_2$, $\tilde{\theta}_2$, and $\theta_{2,\{1\}}$.
\end{itemize}
Evidently, 6 different mappings and 6 different sets of parameters are required.
\end{example}

For general $m$, when estimating all CMI values of the form $I(Y;f_i(X_i)|F(X_S))$, $2^{m-1}$ individual mappings $f_i$ are needed for $X_i$ when there is no limit on the degree, each corresponding to a conditioning set $S\subseteq[m]\setminus {i}$. Note that separate mappings are required for the conditioning sets as well. Moreover, a similar number of conditional distributions (in the form of $q_{\theta_i}$ and $q_{\theta_{i,S}}$) are required. This is clearly intractable, even for small $m$. 

A naive solution to this issue is to maximize the likelihood of only a single surrogate distribution $q_{\theta}(Y|F(V))$ conditioned on all of the feature mappings.
However, consider the following example. Suppose $X_1$ and $X_2$ are identical copies containing the same information about $Y$, such that $I(Y;X_1) = I(Y;X_2) > 0$. Using the naive approach, $f_1$ and $f_2$ are not guaranteed to produce $f_1(X_1)$ and $f_2(X_2)$ satisfying $I(Y;X_1) = I(Y;f_1(X_1)) = I(Y;f_2(X_2))$. This is due to the fact that the best predictor of $Y$ given both $X_1$ and $X_2$ is no better than the best predictor of $Y$ given $X_1$ only. Suppose $f_2\equiv 0$ and $f_1(X_1)=X_1$; no predictor given both $X_1$ and $X_2$ is strictly better than the optimal predictor given $f_1(X_1)$ and $f_2(X_2)$, yet $I(Y;f_2(X_2)) = 0$. Note that the error in estimating $I(Y;X_2)$ also propagates to the estimation of $I(Y;X_1|X_2)$.  This example demonstrates that, while redundancy in the input variables need not be retained for prediction, any redundancy should be retained for CMI estimation. We propose an approach to circumvent the exponential nature of the problem in the next subsection.

\vspace{-2mm}
\subsection{Parameter Sharing and Block-Dropout}\label{subsec:block_dropout}
\vspace{-1mm}

We make two approximations to address the exponential number of required mappings and surrogate conditional distributions, respectively. First, we learn a single set of mappings $f_1$,...,$f_m$ which are shared for all CMI estimations. A multi-objective optimization is required for learning the shared mappings jointly, where the objectives are in the form of \eqref{eq:mi_opt} or \eqref{eq:cmi_opt}. For any subset $S\subseteq [m]$, let $q_{\theta_S}(Y|F(X_S))$ be a surrogate for $p(Y|F(X_S))$. We minimize the weighted sum
\begin{equation}\label{eq:objective_weighted}
    \max_{\substack{
        \theta_S : S\subseteq [m]
        \\ 
        f_i:i\in [m]}}\,\,
        \sum_{S\subseteq [m]}
        P_S \cdot \E \big[\log q_{\theta_S} (Y | F(X_S))\big],
\end{equation}
where $P_S$ is a discrete distribution over $S$.

Although the optimization \eqref{eq:objective_weighted} eliminates the need for having multiple mappings for each feature, an exponential number of surrogate conditional distributions still needs to be estimated.
We propose a randomization-based implementation to circumvent this issue. Our method enables us to share a single set of parameters $\theta$ among all of the surrogate conditional distributions. 
More specifically, 
let $W=[W_1\cdots W_m]$ be a binary random vector of length $m$. Also, denote $F_{W}(V)=\{W_i\cdot f_i(X_i) : i\in [m]\}$. For instance, in the case of $m=2$, 
$F_W(V)=\{W_1f_1(X_1), W_2f_2(X_2)\}$. Note that $W$ is a binary mask for the feature mappings. We propose the following optimization for finding the mappings:
\begin{equation}
\label{eq:objective_dropout}
    \max_{\theta, f_i:i\in [m]}\,\,
    \sum_{w\in \{0,1\}^m}
    P_w \cdot \E \big[\log q_\theta (Y|F_{w}(V), w)\big].
\end{equation}

For every feature subset $S$, there is a corresponding $W$ such that $W_i$ equals 1 if $i\in S$, and 0 otherwise. For this choice of $S$ and $W$, we take $q_\theta(Y|F_{W}(V),W)$ to be the surrogate of $p(Y|F(X_S))$ and set $P_W=P_S$ in \eqref{eq:objective_dropout}.

We next introduce a dropout-based implementation of the program \eqref{eq:objective_dropout}, which we refer to as block-dropout. 
The distribution $q_\theta$ is parameterized by a neural network which takes $F_W(V)$ and $W$ as inputs. In order to obtain samples of $F_W(V)$, for each sample of $F(V)$, a subset of its feature mappings are randomly dropped out, i.e., set to zero, according to the discrete distribution $P_W$. Note that if an individual feature mapping is vector-valued, the entire block of entries corresponding to that mapping are set to zero. Figure \ref{fig:learning_framework} illustrates the proposed approach for learning the feature mappings.

Our proposed approach is reminiscent of dropout regularization, which was originally proposed as a computationally efficient approach for approximate model averaging \citep{SrivastavaHKSS2014}. In fact, our approach is equivalent to a dropout layer applied to the feature mappings when $W$ is chosen to be a random vector with i.i.d. Bernoulli entries and the feature mappings are scalar.
Unlike dropout, block-dropout removes entire feature mappings at a time. Moreover, we instead choose $W$ to be uniformly distributed over length-$m$ binary vectors with number of ones between 1 and $\Delta+1$, due to \eqref{eq:cmi_opt} and the fact that $\Delta$ is the largest conditioning set size in Algorithm \ref{alg:markov_blanket}.

\begin{figure}[t]
\includegraphics[width=0.44\textwidth]{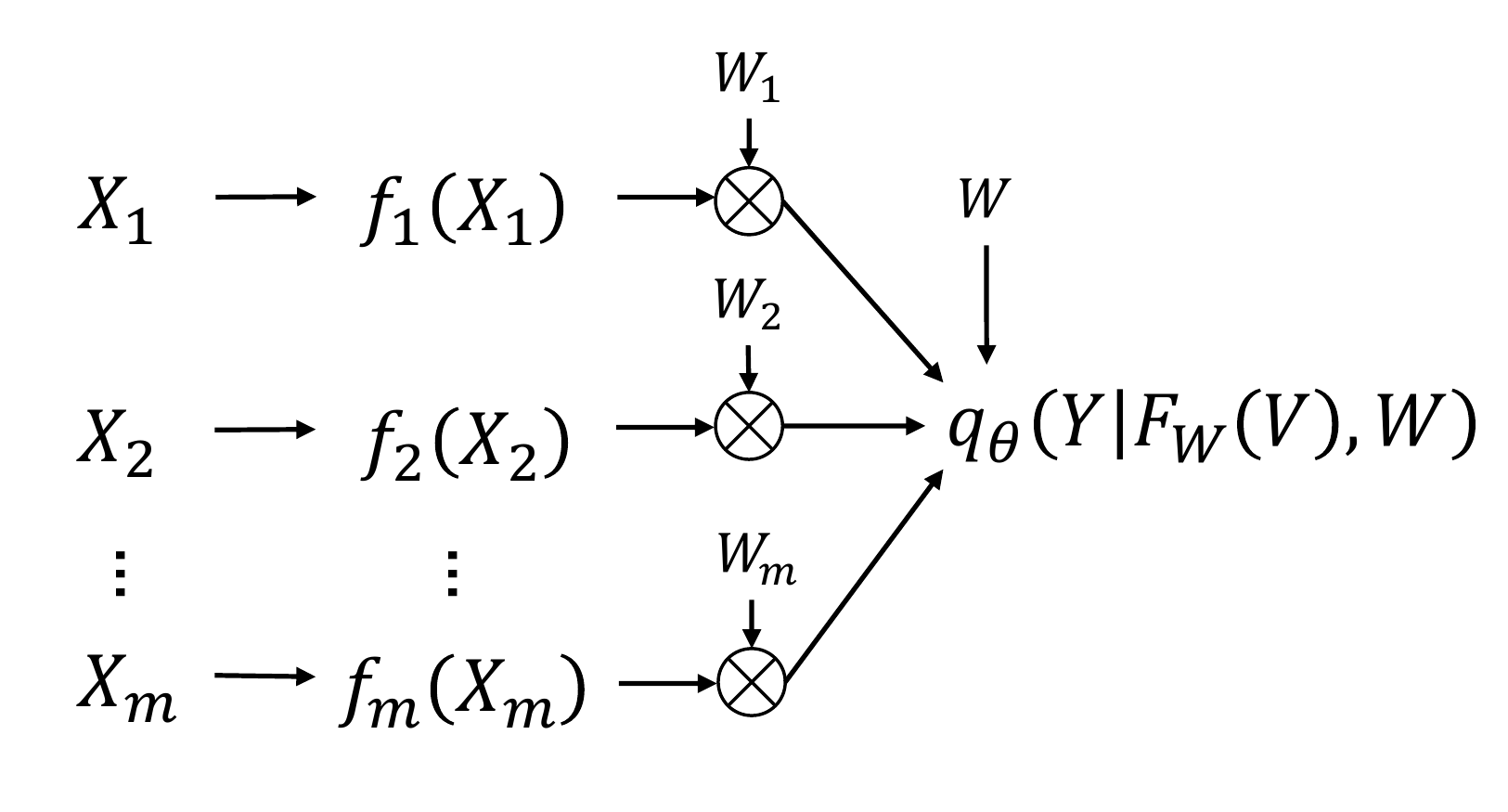}
\centering
\vspace{-3mm}
\caption{Proposed approach for learning the feature mappings. $X_1,...,X_m$ are individually mapped to representations $f_1(X_1),...,f_m(X_m)$. Each mapping $f_i(X_i)$ is dropped with some probability (multiplied with binary $W_i$). The conditional distribution $q_\theta(Y|F_W(V),W)$ is parameterized by the remaining mappings.}
\vspace{-3mm}
\label{fig:learning_framework}
\end{figure}

In the following section, we propose a regularization term that, when added to the objective \eqref{eq:objective_dropout}, ensures that non-parametric estimators such as $k$-NN perform well when given mapped data samples of $F(V)$.

\vspace{-2mm}
\section{INFORMATION EFFICIENCY}\label{sec:info_efficiency_knn}
\vspace{-1mm}

The second step of our approach employs a non-parametric CMI estimator on samples of the learned feature mappings for CI testing. We focus on the $k$-NN approach, although the following discussion is generalizable to other non-parametric estimators which are based on local density estimates around each data point. 

\citet{KraskovSG2004} introduced the $k$-NN MI estimator for the continuous case as an improvement over methods that combine quantization and discrete estimators, which are known to have non-negligible bias \citep{Paninski2003}. The $k$-NN MI estimator was based on the widely-adopted $k$-NN estimator of differential entropy \citep{KozachenkoL1987}, and was later extended to CMI \citep{FrenzelP2007}. In the $k$-NN approach, the joint density is assumed to be constant in a hypercube centered around each data point. The diameter of this hypercube is determined by the point's $\ell_\infty$ distance to its $k$-nearest neighbor, and the sample complexity of the $k$-NN method depends on the rate at which $k$-NN distances shrink as sample size grows \citep{GaoOP2018}. The faster $k$-NN distances converge to zero, the better the sample complexity. 

While the $k$-NN estimator is consistent under mild assumptions \citep{GaoSOP2017,GaoOP2018}, it can exhibit arbitrarily high sample complexity due to its reliance on $\ell_\infty$ distance. Notably, $k$-NN is known to under-perform for distributions with bounded support, a phenomenon known as \emph{boundary bias}. Intuitively, for a point on the boundary of the support, the distance to its $k$-th nearest neighbor shrinks more slowly than the distance for a point in the interior of the support as the sample size grows. As a result, for a fixed $k$, densities on the boundaries tend to be systematically under-estimated.

Boundary bias is an example of a fundamental issue with the $k$-NN approach: the estimation of the density at a point is not accurate unless 
the point is surrounded by sufficiently many similar points. 
To resolve this issue, we propose adding a regularization term to the likelihood objective \eqref{eq:objective_dropout} that ensures that the mappings are suitable for the $k$-NN CMI estimator. 
The main idea is that mapped samples should share similar information about the target variable if and only if they are close in Euclidean distance.

We formalize this intuition by introducing the concept of information efficiency defined as follows. For this definition we use the notion of the Jeffreys divergence \citep{Jeffreys1948} between two distributions $p$ and $q$, which is defined as 
\begin{equation}\label{eq:def_sym_KL}
    D_J(p\|q) = \frac{1}{2}D(p\|q) + \frac{1}{2}D(q\|p),
\end{equation}
where $D(\cdot\|\cdot)$ denotes the KL-divergence.

\begin{definition}[Information Efficiency]\label{def:info_efficiency}
$X$ is \emph{information-efficient} with respect to $Y$ if for any $\epsilon>0$ there exists a $\delta>0$ such that for any realizations $x$ and $x'$ of $X$, $\|x - x'\| < \delta$ if and only if $D_J(p(Y|x)\|p(Y|x')) < \epsilon$, where $\|\cdot\|$ is any norm on the domain of $X$.
\end{definition}

\begin{algorithm}[t]
{\fontsize{10}{16}\selectfont
\begin{algorithmic}[1]
\STATE {\bfseries input:} batch size $n$, structure of $f_1,...,f_m$ and $q_\theta$
\FOR {number of training iterations}
    \STATE{Obtain $n$ data samples $\{(y^{(k)},v^{(k)})\}_{k=1}^n$}
    \STATE{Obtain $n$ samples of $W$: $\{w^{(k)}\}_{k=1}^n$}
    \STATE $\mathcal{J} \gets \frac{1}{n}\sum_{k=1}^n \log q_\theta(y^{(k)}|F_{w^{(k)}}(v^{(k)}),w^{(k)})$
    \FOR {$i=1$ to $m$}
        \STATE $\{(v_i')^{(k)}\}_{k=1}^n\gets \{v^{(k)}\}_{k=1}^n$
        \STATE $\{(v_i'')^{(k)}\}_{k=1}^n\gets \{v^{(k)}\}_{k=1}^n$
        \STATE Shuffle $x_i$ samples of $\{(v_i'')^{(k)}\}_{k=1}^n$
        \FOR {$k$ in $1,...,n$}
            \IF {$w_i^{(k)} = 1$}
                \STATE $(q_i')^{(k)}\gets q_\theta(y^{(k)}|F_{w^{(k)}}((v')^{(k)}),w^{(k)})$
                \STATE $(q_i'')^{(k)}\gets q_\theta(y^{(k)}|F_{w^{(k)}}((v'')^{(k)}),w^{(k)})$
                \STATE $D_i^{(k)}\gets D_J((q_i')^{(k)}\|(q_i'')^{(k)})$
                \STATE $d_i^{(k)}\gets \|f_i((x_i')^{(k)})-f_i((x_i'')^{(k)})\|_2^2$
            \ENDIF
            \STATE $\mathcal{J} \gets \mathcal{J} - \lambda\frac{1}{n} \big\vert d_i^{(k)} - D_i^{(k)}\big\vert$
        \ENDFOR
    \ENDFOR
    \STATE Update $\theta$ and $f_1,...f_m$ to maximize $\mathcal{J}$
\ENDFOR
\STATE \textbf{return} learned mappings $f_1,...,f_m$
\end{algorithmic}}
\caption{Feature Mapping Learning}
\label{alg:training_procedure}
\end{algorithm}

The word ``efficiency'' is used to emphasize that Definition \ref{def:info_efficiency} is a condition for how efficiently the $k$-NN method uses the data points. Section \ref{sec:experiments} provides an example that shows that the sample complexity of $k$-NN for estimating $I(X;Y)$ can be improved by replacing samples of $X$ with those of $f(X)$, where $f(X)$ is both information efficient and a sufficient statistic. We chose Jeffreys divergence as a distance measure between distributions because it is symmetric and can be computed analytically in many cases.

\begin{figure*}[t]
\centering
\begin{subfigure}{.25\textwidth} 
  \centering
  \includegraphics[width=.73\textwidth]{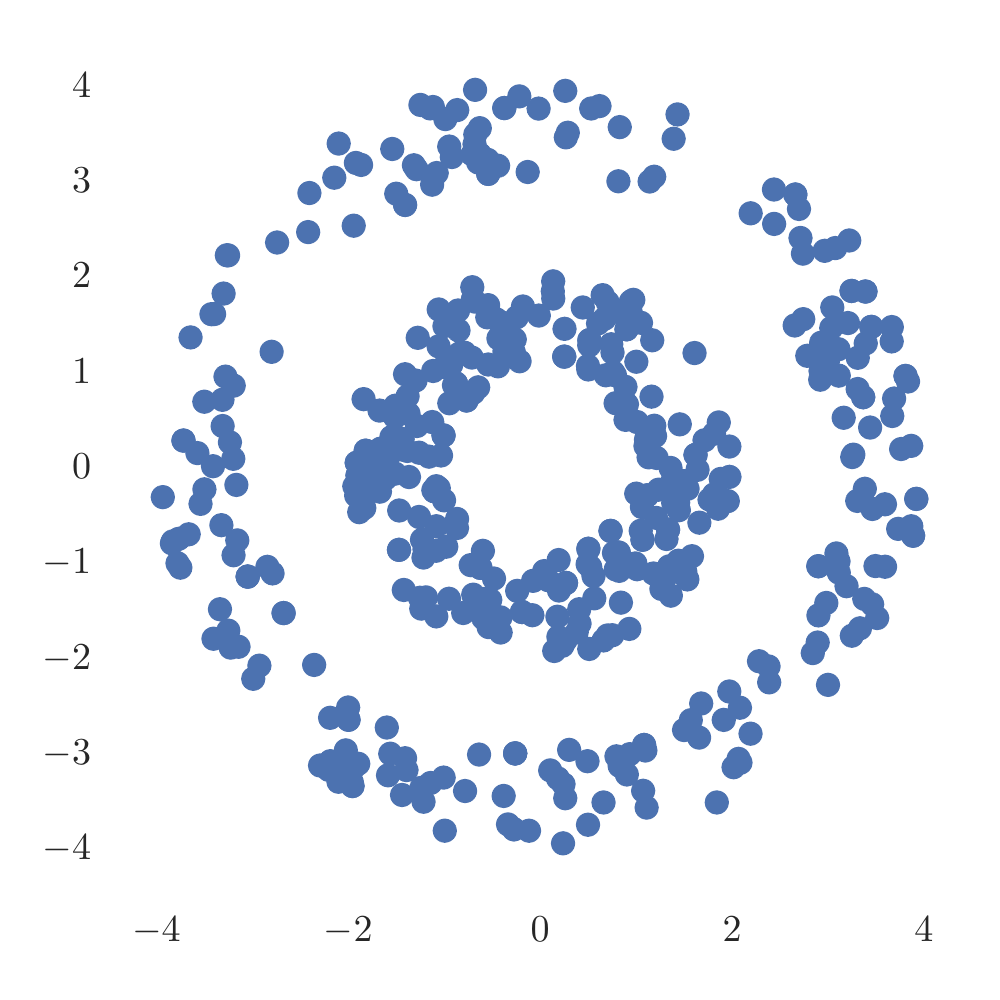}
  \caption{500 samples of $X$}
  \label{fig:bullseye_data}
\end{subfigure}
\begin{subfigure}{.25\textwidth}
  \centering
  \includegraphics[width=0.85\textwidth]{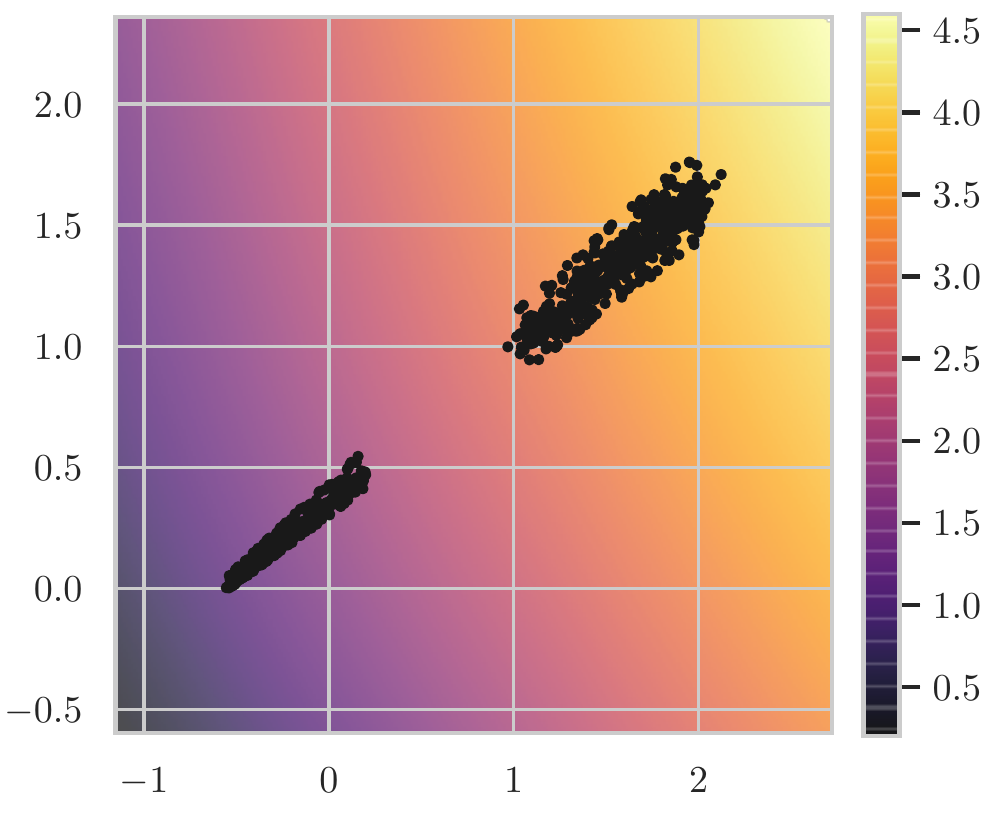}
  \caption{Samples of $f_{\text{regularized}}(X)$}
  \label{fig:z_regularized}
\end{subfigure}
\begin{subfigure}{.25\textwidth}
  \centering
  \includegraphics[width=0.8\textwidth]{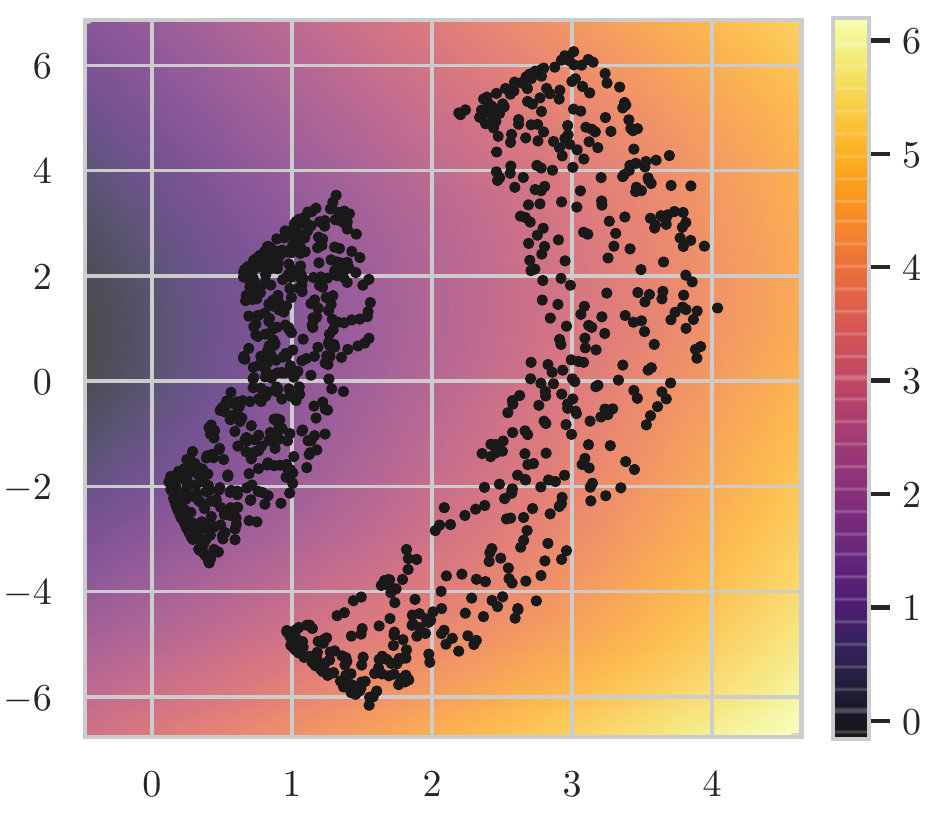}
  \caption{Samples of $f_{\text{nominal}}(X)$}
  \label{fig:z_nominal}
\end{subfigure}
\vspace{-2mm}
\caption{Bullseye dataset ($\epsilon=0.3$). (a) illustrates the support of $X$. Mapped samples of $X$ produced by the proposed method with (b) and without (c) the proposed regularization are plotted over heatmaps of $\E_{q_\theta}[Y|f_{\text{regularized}}(X)]$ and $\E_{q_\theta}[Y|f_{\text{nominal}}(X)]$, respectively. Samples of $f_{\text{regularized}}(X)$ which lie on constant contours are mapped close together.
}
\vspace{-3mm}
\end{figure*}

To produce statistics $f_1(X_1),...,f_m(X_m)$ which exhibit the information efficiency property with respect to $Y$, we propose a 
regularization term that maps dissimilar mapped samples apart while pulling similar samples together. For each feature $X_i$, let $X'_i$ and $X_i''$ be two i.i.d. random samples of $X_i$ and define the modified feature sets
\begin{equation*}
\begin{aligned}
    V_i'  &\vcentcolon= \{X_1,...,X_{i-1},X_i',X_{i+1},...,X_m\},\\
    V_i'' &\vcentcolon= \{X_1,...,X_{i-1},X_i'',X_{i+1},...,X_m\}.
\end{aligned}
\end{equation*}
Furthermore, we define the conditional distributions
$q_i' \coloneqq q_\theta(Y|F_W(V_i'),W)$ and $q_i'' \coloneqq q_\theta(Y|F_W(V_i''),W)$. For each $i$, consider the regularization term
\begin{equation*}
    R_i = \Big\vert \|f_i(X_i')-f_i(X''_i)\|_2^2 - D_J(q_i'\|q_i'') \Big\vert.
\end{equation*}
Minimizing $R_i$ encourages $f_i$ to exhibit the information efficiency property with respect to $Y$. We minimize
\begin{equation}\label{eq:regularization}
    R = \sum_{i=1}^m W_i R_i, 
\end{equation}
where $W_i$ is the $i$-th entry of the block-dropout variable $W$.
For each $i$, $R_i$ is minimized only if $W_i=1$ ($X_i$ is not dropped out).
In summary, we maximize
\begin{equation}
\label{eq:objective_dropout_regularized}
    \max_{\theta, f_i:i\in [m]}\,\, \E \big[\log q_\theta(Y|F_W(V),W) - \lambda R \big],
\end{equation}
where $\lambda$ is a regularization hyperparameter and the expectation is taken over $Y$, $V$, $W$, and $X_i'$, $X_i''$ for all $i$. In practice, for each $i$ we obtain minibatch samples of $X_i'$ and obtain $X_i''$ by shuffling the samples of $X_i'$. The learning algorithm for optimizing \eqref{eq:objective_dropout_regularized} is presented in Algorithm \ref{alg:training_procedure}. Our regularization technique is conceptually related to contrastive loss functions for representation learning, e.g. \citet{HadsellCL2006}. However, \eqref{eq:regularization} also serves to separate points containing dissimilar information, a characteristic important for $k$-NN performance.

\vspace{-2mm}
\section{EXPERIMENTS}\label{sec:experiments}
\vspace{-1mm}

In this section, we evaluate our proposed method on both synthetic and real data. For all experiments, we use the $k$-NN estimator in the second step of our proposed approach. In the case that the target variable $Y$ is discrete, we use a modification of the $k$-NN approach for discrete-continuous mixtures \citep{GaoSOP2017}. Additional details regarding the experiments can be found in the Supplementary Materials.

\begin{figure*}[t]
\centering
\begin{subfigure}{.35\textwidth} 
  \centering
  \includegraphics[width=\textwidth]{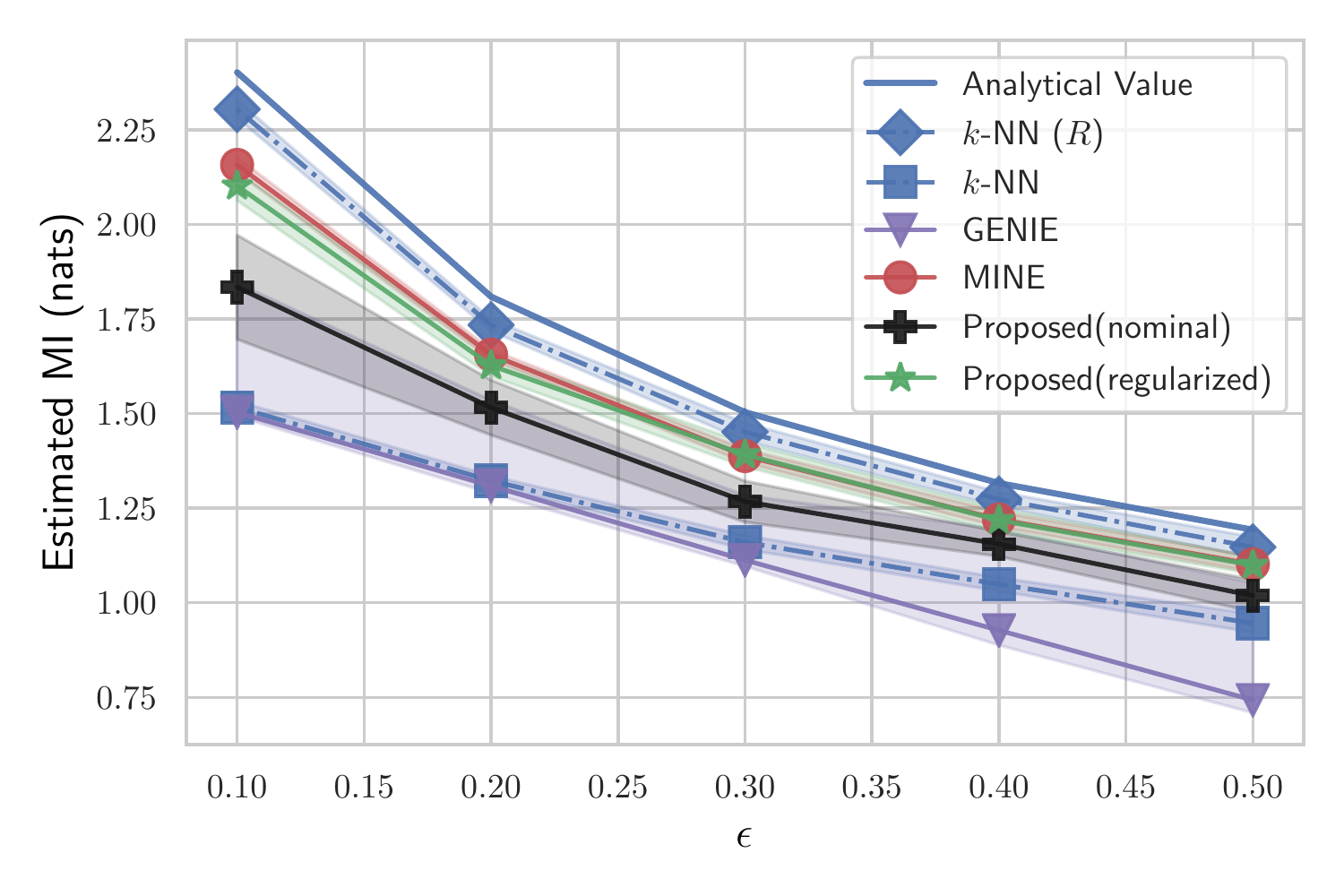}
  \vspace{-2mm}
  \caption{Performance vs. $\epsilon$ for sample size 2000}
  \label{fig:bullseye_performance_epsilon}
\end{subfigure}
\begin{subfigure}{.35\textwidth}
  \centering
  \includegraphics[width=\textwidth]{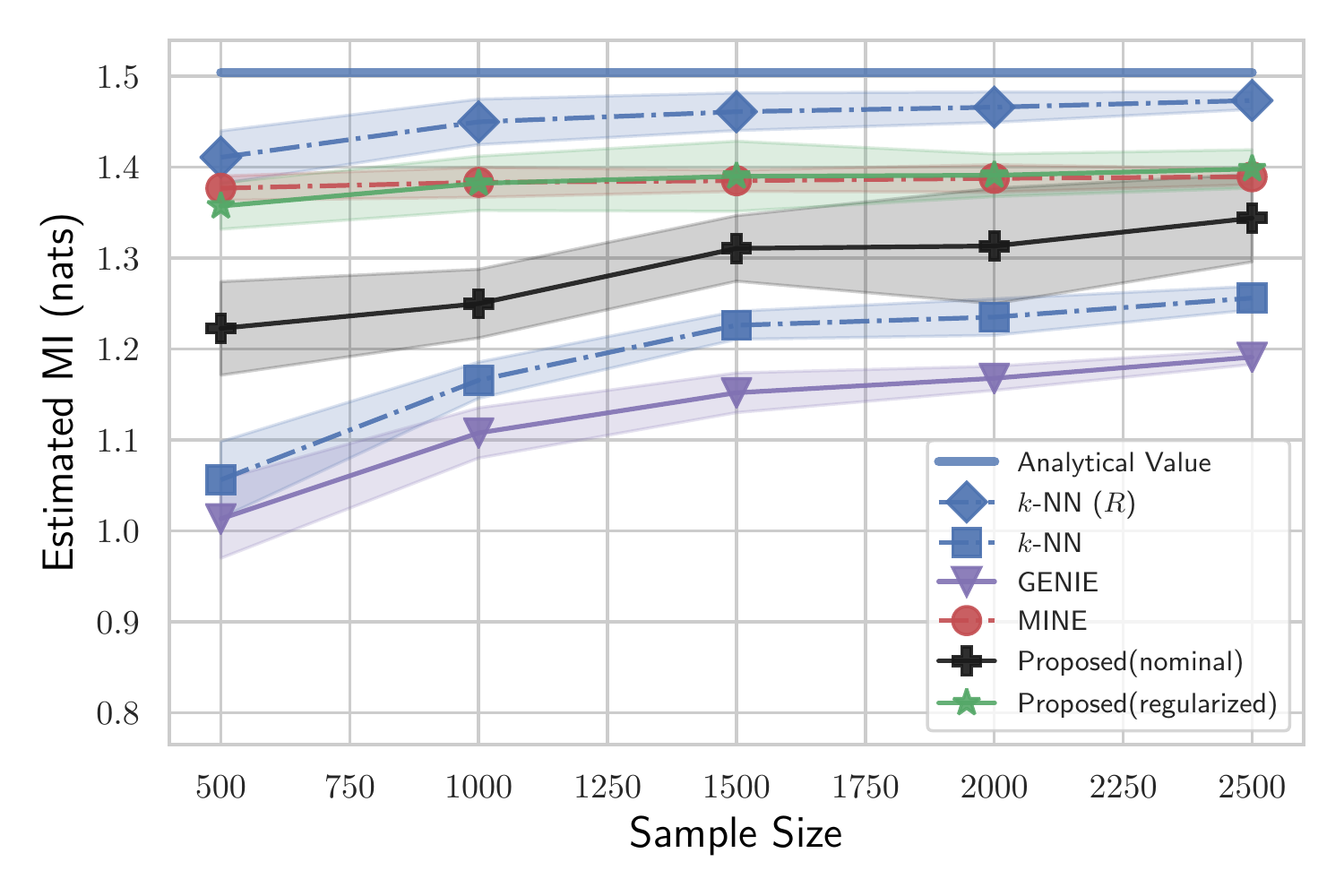}
  \vspace{-2mm}
  \caption{Performance vs. sample size $n$}
  \label{fig:bullseye_performance_n}
\end{subfigure}
\vspace{-2mm}
\caption{Performance comparison of MI estimators on the 2D Bullseye Dataset.}
\vspace{-3mm}
\end{figure*}

\vspace{-2mm}
\subsection{2D Bullseye Dataset}\label{subsec:2d_bullseye}
\vspace{-1mm}

Consider random variables $R\sim\text{Unif}\{[1,2]\cup[3,4]\}$, $\Theta\sim\text{Unif}[0,2\pi]$, and $N\sim\text{Unif}[-\epsilon, \epsilon]$ for $0\leq\epsilon\leq 0.5$. Let $X=(R\cos{\Theta},R\sin{\Theta})$ and $Y=R+N$. Samples of $X$ are plotted in Figure \ref{fig:bullseye_data} for $\epsilon=0.3$; as seen in the figure, the support of $X$ resembles a bullseye in $\R^2$. $X$ is not information efficient with respect to $Y$, since $Y$ only depends on the magnitude of $X$. Therefore, we would expect that non-parametric estimators, such as $k$-NN, systematically underestimate the mutual information $I(X;Y)$. We now evaluate the ability of the regularized objective \eqref{eq:objective_dropout_regularized} to learn an information efficient mapping that improves subsequent $k$-NN MI estimation.

\paragraph{Learning an Information Efficient Mapping.}
Observe that $X$ does not satisfy the information efficiency property with respect to $Y$, since any two $X$ values with the same radial distance from the origin carry the same information about $Y$. Furthermore, the geometry of the support of $X$ introduces a boundary bias. On the other hand, the statistic $R=\|X\|_2$ is information efficient with respect to $Y$, and $I(X;Y)=I(R;Y)$. This value of this MI is derived in the Supplementary Materials. 

In order to compare the effect of the proposed regularization term \eqref{eq:regularization}, we learn two mappings $f_{\text{regularized}}$ and $f_{\text{nominal}}$ by optimizing \eqref{eq:objective_dropout_regularized} and \eqref{eq:objective_dropout} respectively, i.e., $f_{\text{nominal}}$ is learned by optimizing only the likelihood term. The surrogate distribution $q_\theta(Y|F_W(V),V)$ is chosen to be a normal distribution parameterized by three-layer feedforward neural networks, the mappings are also learned using three-layer feedforward neural networks, and the regularization coefficient in \eqref{eq:objective_dropout_regularized} was chosen to be 0.1.
Mapped samples are shown in Figures \ref{fig:z_regularized} and \ref{fig:z_nominal}, respectively. In both cases, mapped samples are plotted over a heatmaps of $\E_{q_\theta}[Y|f_{\text{regularized}}(X)]$ and $\E_{q_\theta}[Y|f_{\text{nominal}}(X)]$, respectively.\footnote{Note that since block-dropout is not applicable for the MI estimation task, the binary mask $W$ is omitted.}
Intuitively, the proposed regularization term pulls points along constant contours together. Unlike the samples of $f_{\text{nominal}}(X)$, samples of $f_{\text{regularized}}(X)$ are clustered around a single line resembling an embedding of $R$ in $\R^2$.

\begin{figure*}[t]
\centering
\begin{subfigure}{.4\textwidth} 
  \centering
  \includegraphics[width=0.5\columnwidth]{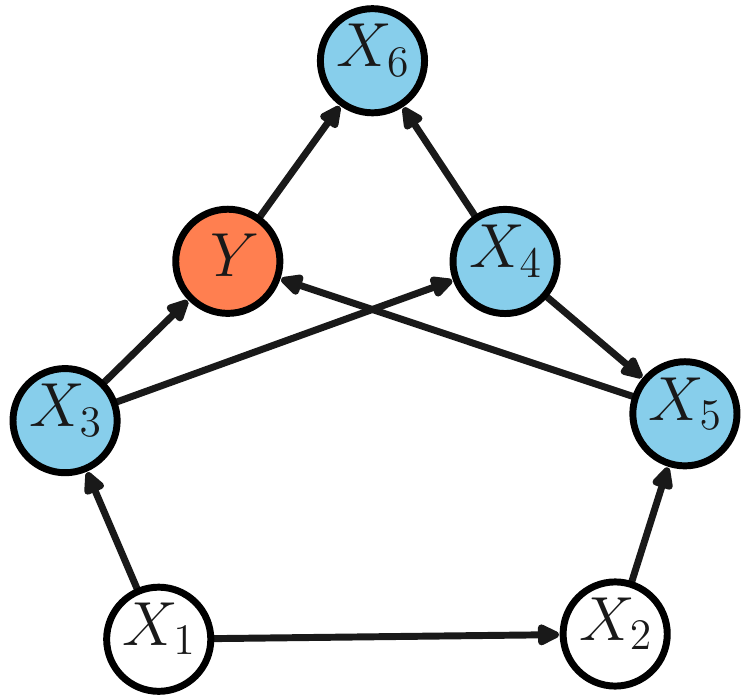}
  \vspace{-2mm}
  \caption{Graphical representation of the target variable $Y$ and its Markov blanket: $X_1,X_3,X_5$, and $X_6$.}
  \label{fig:ci_test_dag}
\end{subfigure}
\begin{subfigure}{.4\textwidth}
  \centering
  \includegraphics[width=0.8\textwidth]{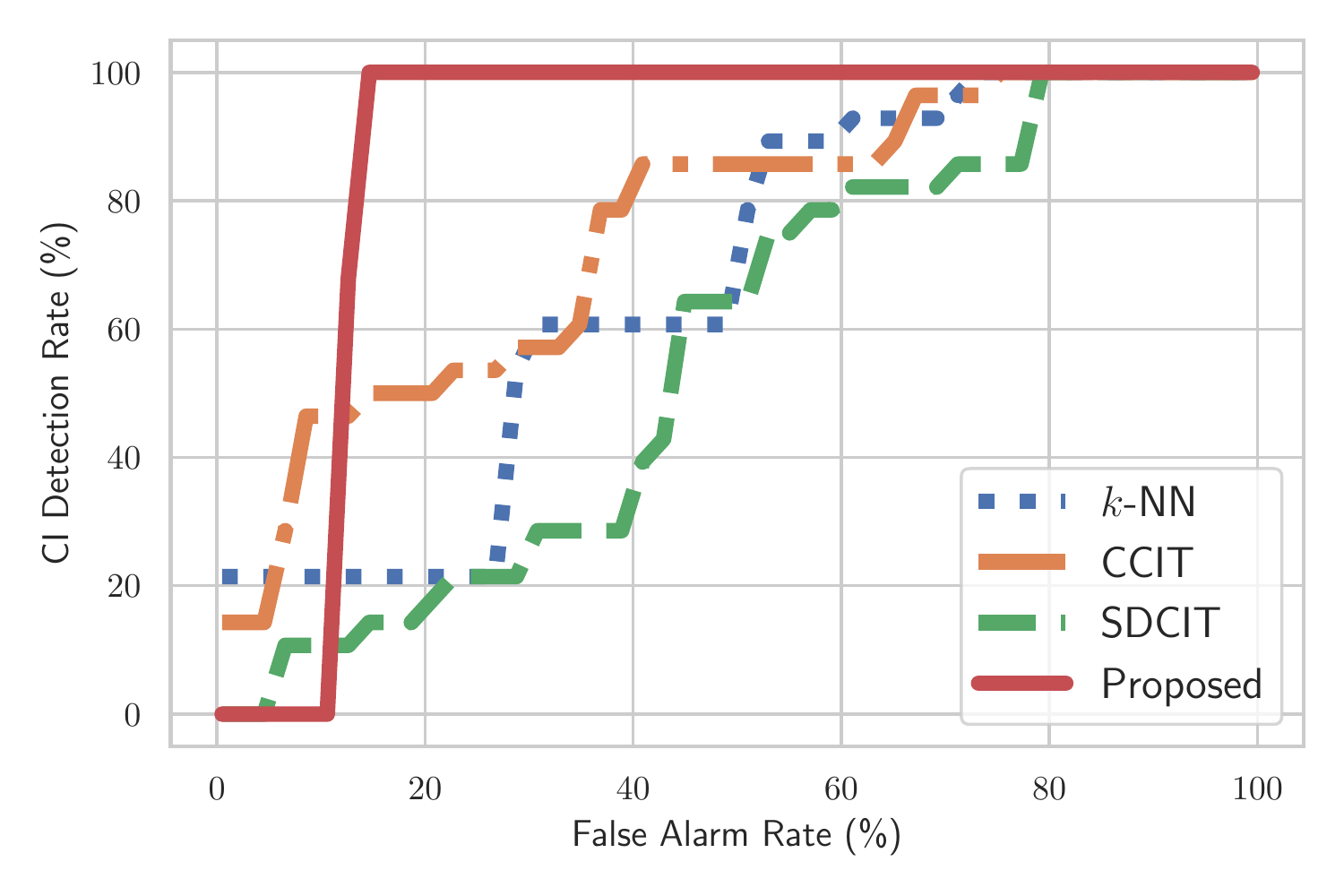}
  \vspace{-2mm}
  \caption{Performance comparison of CI tests.}
  \label{fig:ci_test_roc}
\end{subfigure}
\vspace{-2mm}
\caption{CI Testing experiment with the 3D Bullseye Dataset.}
\end{figure*}

\paragraph{Estimation of $\boldsymbol{I(X;Y)}$.}
We now compare the performance of five estimators of $I(X;Y)$: the $k$-NN method \citep{KraskovSG2004}, the weighted kernel density estimator GENIE \citep{MoonKH2017}, the model-based MINE estimator \citep{BelghaziIASOYCH2018}, and the proposed method, with (regularized) and without (nominal) the proposed regularization term \eqref{eq:regularization}. In addition, we also compare with the the performance of the $k$-NN estimator given samples of $R$ instead of samples of $X$, which is denoted by $k$-NN $(R)$. Figures \ref{fig:bullseye_performance_epsilon} and \ref{fig:bullseye_performance_n} compare the performance of those estimators across different values of $\epsilon$ and sample size, respectively.

Due to the generating model of $Y$, $R$ is an information efficient statistic for estimating the MI, and indeed $k$-NN $(R)$ consistently gives the best performance. Due to the geometry of the distribution, the two non-parametric estimators $k$-NN and GENIE systematically underestimate the MI. The proposed method improves the performance of $k$-NN, especially with the regularization term. The performance improvement can be explained by Figures \ref{fig:z_regularized} and \ref{fig:z_nominal}. Samples of $f_{\text{regularized}}$ and $f_{\text{nominal}}$ containing similar information about $Y$ under the surrogate distribution are clustered more closely together, although this effect is more pronounced in the regularized case.

The proposed (regularized) method performs comparably to the model-based MINE method, suggesting that the gap in performance between $k$-NN and model-based estimators can be bridged using a learned information-efficient statistic. For this problem, knowledge of the joint distribution of $X$ and $Y$ can be used to find a statistic $R$ that improves $k$-NN performance analytically. For the general case, our proposed method provides a framework for learning an appropriate statistic.

\vspace{-2mm}
\subsection{3D Bullseye Dataset}\label{subsec:bullseye_cit}
\vspace{-1mm}

This experiment evaluates the proposed method for CI testing. Consider the graphical model in Figure \ref{fig:ci_test_dag}. The target variable $Y$ and features $X_1,...,X_6$ were generated according to the following generalization of the bullseye example of Subsection \ref{subsec:2d_bullseye}. Let $X_1$ be uniformly distributed on the surface of the sphere centered around the origin with radius uniformly distributed according to $\text{Unif}\{[1,2]\cup[3,4]\}$. For $i\ne 1$, $X_i$ is also uniformly distributed on the surface of a sphere. The radius is given by the average of the magnitudes of the parents of $X_i$, plus noise $N_i$. $Y$ is the average of $\|X_3\|_2$ and $\|X_5\|_2$, plus noise $N_Y$. The noise variables $N_i$ and $N_Y$ are all i.i.d. with distribution $\text{Unif}[-\epsilon, \epsilon]$.

We compare the performance of four different CI tests for identifying a set of CI relationships involving the target variable $Y$: the proposed method, the $k$-NN-based CI test \citep{Runge2018}, the kernel-based SDCIT \citep{LeeH2017}, and the model-based CCIT \citep{SenSSDS2017} using gradient-boosted decision trees. For unconditional independence tests, the Hilbert Schmidt Independence Criterion (HSIC) \citep{GrettonFTSSS2007} is used in place of SDCIT. The statistical significance of the proposed and $k$-NN CI tests are assessed using the nearest-neighbor permutation-based shuffling method proposed by \citet{Runge2018}. The design choices for the proposed method were chosen similarly as in Subsection \ref{subsec:2d_bullseye}. The feature mappings were chosen to be one-dimensional.

Figure \ref{fig:ci_test_roc} illustrates receiver operating curves (ROC) characterizing the CI tests' ability to correctly declare CI using 6000 samples with $\epsilon=0.3$. Overall, the proposed method shows the strongest performance, since it learns information-efficient scalar sufficient statistics for each feature. As expected, of the other three methods, the model-based CCIT method shows the best performance. The geometry of the features' distributions limits the effectiveness of the $k$-NN and kernel-based non-parametric CI tests, which rely on local density estimates.

\vspace{-2mm}
\subsection{Datacenter Hard Disk Drive Dataset}
\vspace{-1mm}

While the HDD is the workhorse of data storage, its reliability is known to be a weak link \citep{Elerath2009}. As a result, HDD failure prediction using diagnostic time-series data has drawn recent attention \citep{XuWLGL2016}. Accurate failure prediction can reduce the amount of redundant storage needed for ensuring data reliability, thereby reducing datacenter costs. 
Commercial and industrial HDDs report time series records of read error rate, temperature, read/write rate, and other prognostics. The following experiment involves data collected from 10,000 Seagate drives of a single model, made publicly available by Backblaze\footnote{backblaze.com/b2/hard-drive-test-data.html}. A total of 36 features, denoted by $X_1,...,X_{36}$, are collected once daily for a period of 90 days. The task is to predict a binary variable $Y$ which indicates whether or not a drive fails at the end of the period; around a quarter of the drives fail. Additional details are presented in the Supplementary Materials.

We chose $q_\theta(Y|F_W(V),W)$ to be a Bernoulli distribution, and the feature mappings $f_i: 1\le i \le 36$, were parameterized by recurrent neural networks, with each $f_i(X_i)\in\R^3$. The maximum conditioning set size $\Delta$ was set to $3$, and the regularization coefficient was taken to be $\lambda=0.1$. For evaluation, two fresh predictors of $Y$ were trained using all 36 and the selected time series features. In addition, we also compared with an approach based on $\ell_1$ regularization, similar to the approach of \cite{LiCW2015}. Each feature was scaled by a trainable weight before being fed into the model; the $\ell_1$ norm of those weights was regularized. Figure \ref{fig:backblaze_roc} compares ROCs of the three resulting models using held-out test data.

The three predictors perform comparably on the test set, although the proposed method was able to find a much smaller feature set than the $\ell_1$ approach. The proposed method selected 7 features: current pending sector, load cycle, and seek counts, logical blocks read and written, power on hours, and temperature. Those features reflect measurable forms of physical wear. HDDs rely on a moving, mechanical actuator to read and write data from a magnetic disk. Therefore, it is expected for logical block read and write counts to be useful for failure prediction. Similarly, the current pending sector count is a measure of the number of memory units, or sectors, that have become unreadable. The seek and load cycle counts record how often the fragile mechanical actuator is repositioned and re-started,  respectively, and are expected to be early indicators of mechanical failures. Finally, drive age and ambient operating temperatures are known correlates of drive failures \cite{Elerath2009}.

\begin{figure}[t]
  \centering
  \includegraphics[width=0.9\columnwidth]{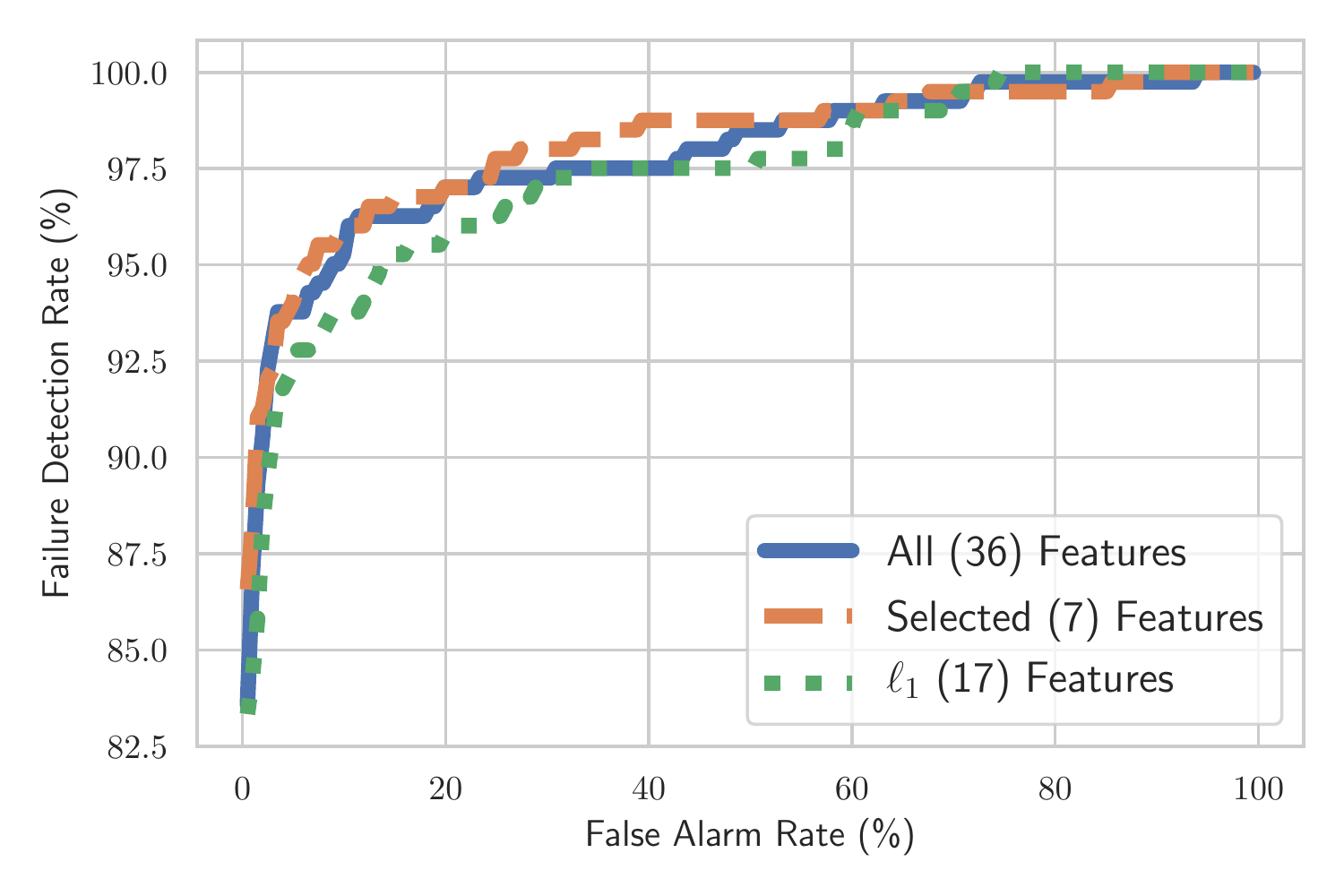}
  \vspace{-2mm}
  \caption{HDD dataset feature selection performance.}
  \vspace{-3mm}
  \label{fig:backblaze_roc}
\end{figure}

 \vspace{-2mm}
\section{CONCLUSION}\label{sec:conclusion}
 \vspace{-1mm}

The need for efficient CI testing in high dimensions has limited the applicability of Markov blanket feature selection. Non-parametric methods, such as the CI test based on the $k$-NN CMI estimator, are efficient and robust for low-dimensional features but struggle with the curse of dimensionality. On the other hand, while model-based CI testing methods have demonstrated success in high dimensions, their computational cost is prohibitive.

We proposed a novel two-step approach to performing the CI tests required for Markov blanket feature selection that leverages both the tractability of non-parametric methods and the representation power of model-based methods. First, low-dimensional feature mappings are simultaneously learned for each feature; the $k$-NN-based CMI estimator is subsequently applied to the learned feature maps for CI testing. The mappings preserve all the information required for CMI estimation, and are trained using a novel regularization term that can improve the performance of any non-parametric CMI estimator or CI test that is based on local density estimation. The performance of our proposed approach was evaluated on synthetic data as well as a real dataset of datacenter hard disk drive failures. We believe that the proposed approach can be extended to improve setups such as learning graphical models and causal structure learning when CI testing is a bottleneck due to high dimension.

\vspace{-2mm}
\subsubsection*{Acknowledgements}
\vspace{-1mm}
This material is based on work supported by NSF CNS 16-24811 -- CAEML and its industry members, NSF CCF 1704970, ONR grant W911NF-15-1-0479, and DARPA under the LwLL program.

\vspace{-2mm}
\subsubsection*{References}
\vspace{-1mm}
\bibliographystyle{apalike}
\renewcommand{\bibsection}{}
\bibliography{augmented_knn}

\newpage
\onecolumn

\begin{center}
{\LARGE \bf Model-Augmented Conditional Mutual Information Estimation for Feature Selection\\~\\Supplementary Materials}
\end{center}

\vskip 0.5in

\appendix

\section{Bayesian Networks}
\label{app:sec:MF}
Let $\G$ be a directed acyclic graph (DAG) in which each vertex represents one of the variables from $\mathbf{X}$. $X_i$ is called a parent of $X_j$ if we have the edge $X_i\rightarrow X_j$, and a child of $X_j$ if we have the edge $X_j\rightarrow X_i$ in $\mathcal{G}$. The descendants of $X_j$ is the set of $X_i$ such that a directed path exists from $X_j$ to $X_i$. Every variable is assumed to be a descendent of itself.
\begin{definition}
For DAG $\mathcal{G}$ and distribution $P$ on the set of variables $\mathbf{X}$,
the pair $(\mathcal{G},P)$ is called a Bayesian network if each variable in $\G$ is independent of its non-descendents given its parents in $P$ (referred to as the Markov condition).
This leads to the following factorization of the joint distribution.
\[
P_{{X}} = \prod_{X_j\in{X}} P_{X_j |  \textit{Pa}(X_j)},
\]
where $\textit{Pa}(X_j)$ denotes the set of the parents of $X_j$.
\end{definition}
Let $\mathcal{I}(P)$ represent the set of all conditional independence relationships in $P$, and $\mathcal{I}(\mathcal{G})$ represent the set of all d-separation relations\footnote{See \cite{Pearl2009} for the definition of  d-separation.} in $\G$. 
\begin{definition}[Faithfulness]
The distribution $P$ is \emph{faithful} to structure $\mathcal{G}$ if for any two variables $X_i$, $X_j$, and any subset of variables $X_S$, we have
\[
(X_i~\textit{d-sep}~ X_j|X_S)\in\mathcal{I}(\mathcal{G})\textit{ if } 
(X_i\indep X_j|X_S)\in\mathcal{I}(P).
\]
\end{definition}
If the Markov and faithfulness conditions hold, $\mathcal{G}$ is called a perfect I-map for distribution $P$.
In general, perfect I-map is not unique. For instance, for a joint distribution $P$ on variables $\{X_1,X_2,X_3\}$, such that $\mathcal{I}(P)=\{(X_1\indep X_3|X_2)\}$ all three DAGs $\mathcal{G}_1:X_1\rightarrow X_2\rightarrow X_3$,
$\mathcal{G}_2:X_1\leftarrow X_2\rightarrow X_3$, and
$\mathcal{G}_3:X_1\leftarrow X_2\leftarrow X_3$ are perfect I-maps.
\begin{proposition}
Under Markov and faithfulness assumptions, the Markov blanket of each variable is unique and consists of its parents, children, and coparents.
\end{proposition}

\section{Connections to Causal Feature Selection}\label{app:CFS}

Causation is a topic of interest in many applied science disciplines. Knowing the causal structure among the set of variables in the system, enables us to predict the consequence of actions and interventions in the system and answer to counterfactual questions. Moreover, it provides us with a better understanding regarding the way a complex system works by explaining the interactions among the components of the system.

As mentioned in Appendix \ref{app:sec:MF}, a Bayesian network which is an I-map for a given distribution is not unique. \emph{Causal Bayesian network} is a Bayesian network in which directed edge $X_i\rightarrow X_j$ implies that $X_i$ is a direct cause of $X_j$. 
The set of direct causes (parents) of variable $X_j$ is denoted by $\textit{DC}(X_i)$. 
In the causal nomenclature, the mechanism which takes the direct causes of a variable $X_i$ as the input and outputs 
variable $X_i$, i.e., $P_{X_i |  \textit{DC}(X_i)}$ is called the causal module corresponding to variable $X_i$. According to the principle of independent causal mechanisms, if there are no latent variables in the system, the causal modules are independent of each other \cite{peters2017elements}. 

Knowing the causal modules enables us to perform prediction under distribution shift: 
Suppose in domain 1, we train a predictor for target variable $Y$ which takes the subset of features $X_S$ as the input. Now, if in a second domain the distribution of $X_S$ varies, but the causal module corresponding to $Y$ remains fixed (based on the principle of independent causal mechanisms, this is possible as the causal modules are independent), then if $X_S$ only contains direct causes of $Y$, then the trained predictor can still be used in the second domain. However, if $X_S$ also contains children of $Y$, as is the case for, say, a feature selection based on Markov blanket, then $P_{Y|X_S}$ in the second domain will not necessarily remain the same as the one in the first domain.
Therefore, a feature selection scheme which chooses the direct causes of the target variable, enables unsupervised domain adaptation. This scheme is referred to as \emph{causal feature selection}.

Note that in general, performing interventions is required to identify the direct causes of the target variable, and hence, causal feature selection in general is not feasible. However, as mentioned in Section \ref{sec:MB_feature_selection}, in many applications, the target variable cannot be the \emph{cause} of any of the features in the system. For example, in our example of HDD failure prediction, $Y$ represents the health state of the HDD at the end of a time interval, and as such cannot be the cause of any of the features as it cannot be the cause of past events. In this case, in the causal Bayesian network, the target variable does not have any children. Therefore, the Markov blanket of the target variable will only contain its direct causes. Hence, Markov blanket feature selection is equivalent to causal feature selection. We summarize this argument as follows:
\begin{assumption}
\label{ass:barren}
The target variable $Y$ does not have any children in the causal Bayesian network.
\end{assumption}
\begin{theorem}
Under the Markov and faithfulness conditions and Assumption \ref{ass:barren}, a feature selection scheme which chooses features based on Definition \ref{def:markov_blanket}, outputs the direct causes of the target variable.
\end{theorem}

\section{Proof of Theorems}

\begin{proof}[Proof of Theorem \ref{theorem:mi_lower_bd}]
    The error is given by
    \begin{align}
        I(Y;X_i) - I(Y;f_i(X_i))
            &= -h(Y|X_i) + h(Y|f_i(X_i)) \nonumber \\
            &= -h(Y|X_i) - \E_{XY} [\log q_{\theta_i}(Y|f_i(X_i))] 
            - D(p(Y|f_i(X_i)\|q_{\theta_i}(Y|f_i(X_i)) \nonumber \\
            &\le -h(Y|X_i) - \E_{XY} [\log q_{\theta_i}(Y|f_i(X_i))].\label{eq:ineq_mi_lb}
    \end{align}
    The inequality \eqref{eq:ineq_mi_lb} follows from the non-negativity of the KL-divergence. Since the right hand side of \eqref{eq:ineq_mi_lb} is an upper bound on the error and the (differential) conditional entropy $h(Y|X_i)$ is not a function of $f_i$ or $\theta_i$, maximizing $\E_{XY} [\log q_{\theta_i}(Y|f_i(X_i))]$ over $f_i$ or $\theta_i$ minimizes an upper bound on the error.
\end{proof}

\begin{proof}[Proof of Theorem \ref{theorem:cmi_consistency}]
By the chain rule of mutual information, we have $I(Y;X_i|X_S) = I(Y;X_i,X_S) - I(Y;X_S)$, and similarly $I(Y;f_i(X_i)|F(X_S)) = I(Y;f_i(X_i),F(X_S)) - I(Y;F(X_S))$. Subtracting and applying the triangle inequality gives
\begin{align*}
\big|I(Y;X_i|X_S) - I(Y;f_i(X_i)|F(X_S))\big| 
&\le \big|I(Y;X_S) - I(Y;F(X_S))\big| + \big|I(Y;X_i,X_S) - I(Y;f_i(X_i),F(X_S))\big|\\
&\overset{(a)}{=} \big(I(Y;X_S) - I(Y;F(X_S))\big) + \big(I(Y;X_i,X_S) - I(Y;f_i(X_i),F(X_S))\big).
\end{align*}
The equality $(a)$ follows from the data-processing inequality; each of the two terms inside the absolute values are non-negative. Applying Theorem \ref{theorem:mi_lower_bd} to each of the two terms gives the result.
\end{proof}

\section{Derivation of $I(X;Y)$ for the Bullseye Experiment}
Assume that $\epsilon\leq 0.5$. The mutual information is given by
\begin{align*}
    I(X;Y) 
    &= I(R;Y) \\
    &= h(Y) - h(Y\vert R) \\
    &= h(Y) - h(R+N \vert R) \\
    &= h(Y) - h(N).
\end{align*}
Since $N\sim\text{Unif}[-\epsilon,\epsilon]$, it has differential entropy $h(N)=\log(2\epsilon)$. To find $h(Y)$, first note that $Y$ can be seen as a randomization between the inner and outer rings:
\begin{equation}
    Y = \begin{cases}
        Y_1 & \text{w.p.}\, 1/2 \\
        Y_2 & \text{w.p.}\, 1/2.
    \end{cases}
\end{equation}
Define $Y_1=R_1+N_1$ and $Y_2=R_2+N_2$. $R_1\sim\text{Unif}[0.25, 0.5]$, $R_2\sim\text{Unif}[0.75, 1.0]$, and $N_1$ and $N_2$ are identically distributed copies of $N$. The marginal distribution $p(y_1)$ is given by
\begin{align}
    p(y_1) 
    &= \int_{-\epsilon}^\epsilon p(N=n)p(y_1\vert n) dn \\
    &= \frac{1}{2\epsilon} \int_{-\epsilon}^\epsilon p(R_1=y_1-n) dn \\
    &= \frac{2}{\epsilon} \int_{-\epsilon}^\epsilon \bs{1}\{y_1 \in [0.25+n, 0.5+n] \} dn.
\end{align}
There are three cases, depending on the value of $y_1$. Taking the integral in each case gives
\begin{equation}\label{eq:p_y1}
    p(y_1) = \begin{cases}
    \frac{y_1-0.25+\epsilon}{0.5\epsilon} & \text{if $y_1\in [0.25-\epsilon, 0.25+\epsilon]$} \\
    4 & \text{if $y_1\in [0.25+\epsilon, 0.5-\epsilon]$} \\
    \frac{0.50+\epsilon - y_1}{0.5\epsilon} & \text{if $y_1\in [0.50-\epsilon, 0.50+\epsilon]$}
    \end{cases}
\end{equation}
The marginal $p(y_2)$ can be obtained by shifting $p(y_1)$. $h(Y)$ can now be computed numerically for different values of $\epsilon$ using Eq. \eqref{eq:p_y1} and the fact that
\begin{equation}
    h(Y) = -\int_{-\infty}^\infty \frac{1}{2}p(y_1)\log \bigg(\frac{1}{2}p(y_1)\bigg) dy_1 
    -\int_{-\infty}^\infty \frac{1}{2}p(y_2)\log \bigg(\frac{1}{2}p(y_2)\bigg) dy_2.
\end{equation}

\section{Experiment Details}

\subsection{2D Bullseye}
Both $f_{\text{regularized}}$ nominal $f_{\text{nominal}}$ are learned using three-layer, feedforward neural networks with 8 hidden units and rectified linear activation. The parameters of the surrogate distribution were learned using a feedforward neural network of the same configuration. The regularization coefficient was chosen to be $\lambda=0.1$, and the mappings were trained using 2000 samples.

\subsection{3D Bullseye}
The mappings were each learned using a three-layer feedforward neural network with 32 hidden units and the rectified linear activation. The parameters of the surrogate distribution were learned using a three-layer feedforward neural network with 164 hidden units, again with the rectified linear activation. As before, we used $\lambda=0.1$. The mappings were trained using 6000 examples. For the $k$-NN CI tests, we used a $k$-value of 100, 1000 shuffling instances, and nearest-neighbor permutation $k_{perm}=5$, all following suggested values in \citep{Runge2018}. CCIT was set up using the publicly available code\footnote{https://github.com/rajatsen91/CCIT}, and likewise for SDCIT\footnote{https://github.com/sanghack81/SDCIT}.

\subsection{Datacenter HDD Dataset}
Two layer gated recurrent units with 16 states were used to implement each feature mapping. The surrogate distribution was parameterized by a three-layer feedforward neural network with 72 hidden units and the rectified linear activation. As for the other experiments, we took $\lambda=0.1$. Each feature mapping was chosen to be 3-dimensional, and the maximum conditioning set size was set to $\Delta=3$.

We used data collected from 10,000 Seagate ST4000DM000 HDDs. Table \ref{table:smart_attrs} lists the time-series features reported by each drive. The features are known as Self-Monitoring, Analysis, and Reporting Technology (SMART) attributes. Some of the features are reported in two separate formats: \emph{raw} and \emph{normalized}. The raw features can be considered as raw sensor measurements, whereas the normalized features are quantized and scaled in a manufacturer-specific way in order to be compared between HDD makes and models. Table \ref{table:smart_attrs} lists all of the features used for this experiment. Finally, note that Seek Errors and Seek Count are not typically-reported SMART attributes. These were generated by extracting the top 16 and bottom 32 bits of the raw SMART attribute "Seek Error Rate." Since we noticed that the normalized version of "Seek Error Rate" could be well-approximated by dividing the Seek Error count by the Seek Count, we do not use Seek Error Rates for our experiments. Table \ref{table:smart_attrs} lists all of the features in the dataset, with descriptions.\footnote{https://en.wikipedia.org/wiki/S.M.A.R.T.}

\begin{table}[t]
\caption{SMART Attributes for Seagate's drive model ST4000DM000. * denotes attributes reported in both normalized and raw form. $\dagger$ have only normalized values, and $\ddagger$ only raw values. There are 36 attributes in total.}
\label{table:smart_attrs}
\begin{center}
\begin{tabular}{ll}
\multicolumn{1}{c}{\bf Feature}  &\multicolumn{1}{c}{\bf Description} \\
\hline \\
*Read Error Rate                 &Rate of errors that occurred when reading data from  disk surface \\
*Temperature                     &Device temperature \\
*Airflow Temperature             &Contains same information as the Temperature feature \\
*Reallocated Sectors             &Count of bad memory sectors that have been found and remapped \\
*SATA Downshift Errors           &Number of downshifts of link speed \\
*End to End Errors               &Count of parity errors in the data path to the media via the drive's cache RAM \\
*Load Cycles                     &Count of load/unload cycles into head landing zone position \\
*Command Timeouts                &Count of aborted operations due to HDD timeout \\
*Current Pending Sectors         &Count of sectors waiting to be remapped, due to unrecoverable read errors \\
*Reported Uncorrectable Errors   &Count of errors that could not be recovered using error correction codes \\
*Power On Hours                  &Count of hours in power-on state \\
*Start/Stop Count                &Count of read spindle start/stop cycles \\
*High Fly Writes                 &Number of times a recording head is flying outside its normal operating range \\
\textsuperscript{$\dagger$}Offline Uncorrectable Sectors   &Count of uncorrectable errors when reading/writing a sector \\
\textsuperscript{$\dagger$}Cyclic Redundancy Check Errors  &Count of errors in external data transfer via interface cable  \\
\textsuperscript{$\dagger$}Spin Up Time                    &Average time of spindle spin up, from zero RPM to fully operational \\
\textsuperscript{$\dagger$}Seek Errors                     &Number of seek errors of the magnetic heads \\
\textsuperscript{$\dagger$}Seek Count                      &Total number of magnetic head seek operations\\
\textsuperscript{$\ddagger$}Power Off Retracts              &Count of power-offs or emergency head retractions \\
\textsuperscript{$\ddagger$}Power Cycles                    &Count of HDD power on/off cycles \\
\textsuperscript{$\ddagger$}Head Flying Hours               &Total time positioning the read/write heads \\
\textsuperscript{$\ddagger$}Logical Blocks Written          &Count of logical block addresses written \\
\textsuperscript{$\ddagger$}Logical Blocks Read             &Count of logical block addresses read
\end{tabular}
\end{center}
\end{table}

\end{document}